\documentclass{article}
\usepackage[sectionbib]{natbib}
\usepackage{hyperref}

\usepackage{arxiv}
\usepackage{bm}
\usepackage[utf8]{inputenc} 
\usepackage[T1]{fontenc}    
\usepackage{url}            
\usepackage{booktabs}       
\usepackage{amsfonts}       
\usepackage{nicefrac}       
\usepackage{microtype}      
\usepackage{amsmath}
\usepackage{lipsum}
\usepackage{amsthm}
\usepackage{cleveref}
\usepackage{graphicx,psfrag,epsf}
\usepackage[ruled,vlined]{algorithm2e}
\usepackage{enumitem}
\usepackage{caption}
\usepackage{sectsty, secdot}

\newtheorem{theorem}{Theorem}
\newtheorem{lemma}{Lemma}

\theoremstyle{definition}

\newtheorem{example}{Example}

\DeclareMathOperator*{\argmin}{\arg\!\min}

\title{Distributional Clustering: A distribution-preserving clustering method}

\author{
  Arvind Krishna\\
  H. Milton Stewart School of Industrial \& Systems Engineering\\
  Georgia Institute of Technology\\
  Atlanta, GA 30332 \\
  \texttt{akrishna39@gatech.edu} \\
   \And
 Simon Mak \\
 Department of Statistical Science\\
  Duke University\\
  Durham, NC 27708 \\
  \texttt{sm769@duke.edu} \\
  \AND
  Roshan Joseph \\
  H. Milton Stewart School of Industrial \& Systems Engineering\\
  Georgia Institute of Technology\\
  Atlanta, GA 30332 \\
  \texttt{roshan@gatech.edu} \\
}

\begin{document}
\maketitle

\begin{abstract}
One key use of k-means clustering is to identify cluster prototypes which can serve as representative points for a dataset. However, a drawback of using k-means cluster centers as representative points is that such points distort the distribution of the underlying data. This can be highly disadvantageous in problems where the representative points are subsequently used to gain insights on the data distribution, as these points do not mimic the distribution of the data. To this end, we propose a new clustering method called ``distributional clustering'', which ensures cluster centers capture the distribution of the underlying data. We first prove the asymptotic convergence of the proposed cluster centers to the data generating distribution, then present an efficient algorithm for computing these cluster centers in practice. Finally, we demonstrate the effectiveness of distributional clustering on synthetic and real datasets.
\end{abstract}

\keywords{Cluster prototypes \and Data reduction \and K-means clustering \and Representative points}

\renewcommand{\theequation}{\thesection.\arabic{equation}}
\setcounter{section}{0} 
\setcounter{equation}{0} 
\section{Introduction}
Clustering is widely used in a variety of statistical and machine learning applications. These applications can be roughly split into two objectives \citep{tan2016introduction}: (i) the identification of homogeneous groups within a dataset, and (ii) the summarization of a dataset into ``representative points'' or ``cluster prototypes''. Objective (i) spans a broad range of real-world problems, including the discovery of gene groups which have similar biological functions, the identification of communities in social networks, and so on. Similarly, objective (ii) covers many important problems in the ever-present world of big data, including data summarization (the reduction of a dataset for expensive downstream computations) and data compression (the representation of a dataset by its prototypes). We will focus on the latter objective in this work.

K-means clustering \citep*{gan2007data} is one of the most widely used clustering methods. Let $\{\bm{x}_j\}_{j=1}^N  \subseteq \mathbb{R}^p$ be a $p$-dimensional dataset of $N$ points, and suppose $n < N$ cluster prototypes are desired from this data. K-means clustering chooses these prototypes as the following optimal cluster centers:
\begin{equation}	
\mathcal{D}_{n} := \argmin_{\mathcal{D}: \# \mathcal{D}=n} \frac{1}{N} \sum_{j=1}^{N} \|\bm{x}_j - Q(\bm{x}_j; \mathcal{D})\|_2^2,
\label{eq:kmean_obj_dis}
\end{equation}
where $\mathcal{D} = \{\bm{d}_1, \cdots, \bm{d}_n\}$ are the $n$ cluster centers to be optimized, and $Q(\bm{x}; \mathcal{D}) := \argmin_{\bm{d} \in \mathcal{D}} \|\bm{x}-\bm{d}\|_2$ returns the closest point in $\mathcal{D}$ to $\bm{x}$. In other words, these optimal centers minimize the sum of the squared distances between it and its closest points in the data. The clustering criterion \eqref{eq:kmean_obj_dis} is known as the \textit{within-cluster sum-of-squares criterion} \citep{pollard1981strong}. This k-means approach for generating representative points has been widely used (with minor variations) in statistics (e.g., \cite{dalenius1950problem}, the principal points in \cite{flury1990principal}, the mse-rep-points in \cite{fang1994number}), signal processing \citep{lloyd1982least}, and stochastic programming \citep{heitsch2003scenario}.

Although k-means clustering is an intuitive and easy-to-implement clustering method, one key weakness is that the cluster prototypes in \eqref{eq:kmean_obj_dis} distort the distribution of the underlying data. To demonstrate this, we perform k-means clustering on a $N$ = 100,000-point dataset, simulated from the univariate standard normal distribution, to obtain $n$ = 100 representative points. Figure \ref{fig:motivation} shows the density of the data generating distribution (call this $f$), along with the kernel density estimate of the k-means cluster centers. We see that the latter density is much more heavy-tailed than the former density, which suggests that the k-means cluster prototypes distort the true data distribution $f$. In fact, this distortion can be quantified theoretically: one can show (see \cite{zador1982asymptotic}, \cite{su2000asymptotically} and Theorem 7.5 of \cite{GL2007}) that the density of k-means cluster centers converges to $f^{p/(p+2)}$ as the number of prototype points $n \rightarrow \infty$. In other words, even with a large number of points, the distribution of k-means centers will not converge to the underlying data distribution!
\begin{figure}[t]
\centering
\includegraphics[width=0.65\textwidth]{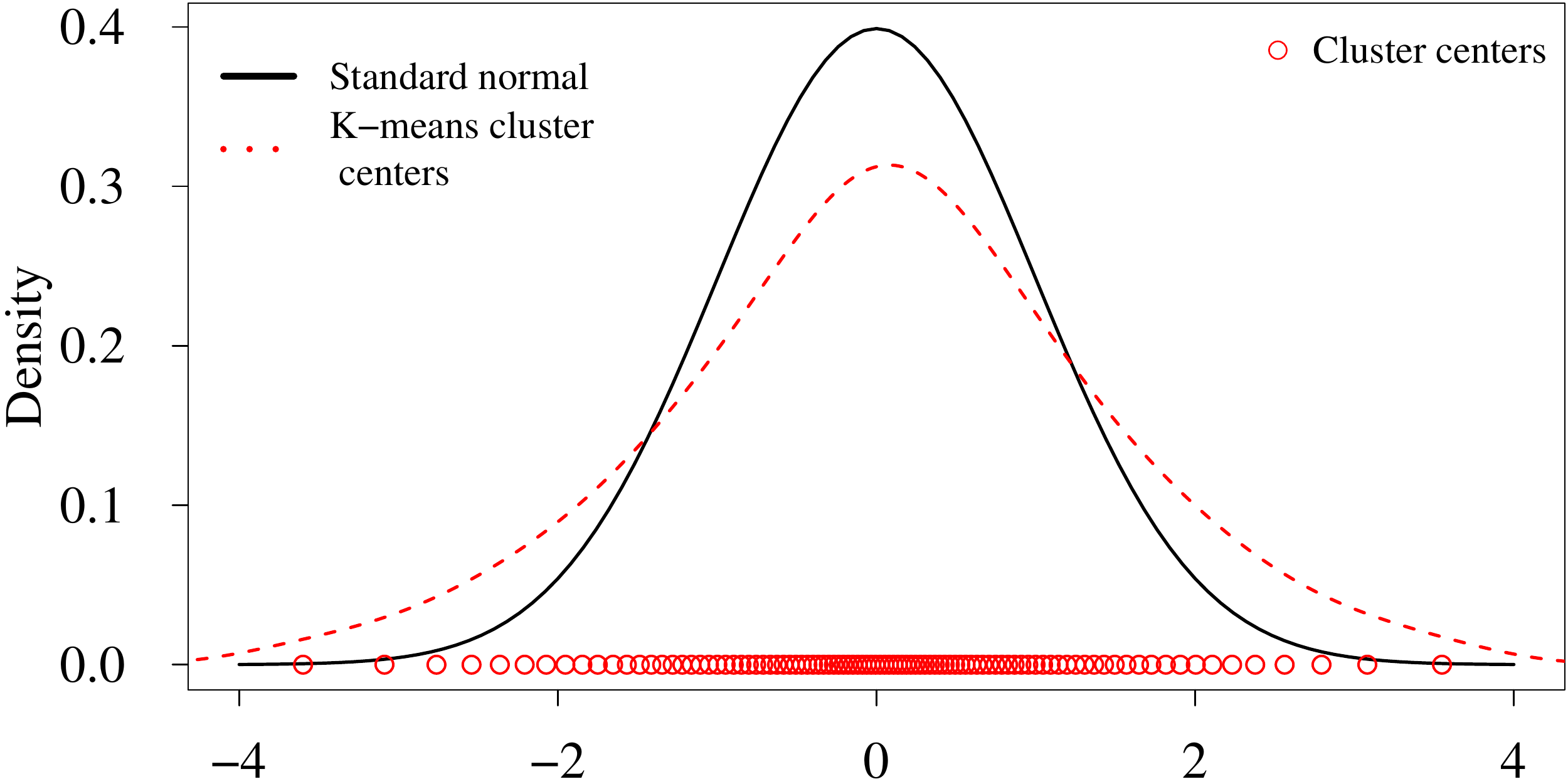}
\caption{Comparison between the data density (solid) and kernel density estimate (dashed) of the k-means cluster centers (circled points). }
\label{fig:motivation}
\end{figure}

This distribution distortion can be detrimental in many applications of cluster prototyping. For example, consider the problem of data summarization, where a large dataset is reduced to a smaller representative dataset, to be used for time-consuming downstream computations. Such computations may involve fitting complex models to data, or conducting expensive experiments at each data point (e.g., in uncertainty quantification \citep{smith2013uncertainty}). Here, a distribution distortion can result in highly biased insights from downstream computations, in that the results using the summarized data may deviate greatly from results using the full data. In this sense, the representative points from k-means clustering may yield undesired and incorrect results when used for data summarization.

To address this, we propose a new clustering method called ``distributional clustering'', whose cluster centers preserve the underlying data distribution. The key idea is to generalize the within-cluster sum-of-squares criterion in \eqref{eq:kmean_obj_dis}, to a new \textit{within-cluster sum-of-powers criterion}, which minimizes the sum of distances to a positive power $k$ within each cluster. We first show that, as power $k \rightarrow 0^+$, the optimal cluster centers from this new clustering criterion indeed converge to the true data-generating distribution as the number of clusters tend to infinity; this addresses the distortion problem of k-means. We then propose a modification of Lloyd's algorithm (a widely-used algorithm for solving the k-means criterion \eqref{eq:kmean_obj_dis}) to compute these distributional clustering centers in practice. Finally, we demonstrate the effectiveness of distributional clustering in several numerical examples. 

\setcounter{equation}{0} 
\section{Distributional clustering}
To address the distribution distortion of k-means cluster centers, we begin by modifying the within-cluster sum-of-squares criterion in (\ref{eq:kmean_obj_dis}), by raising the Euclidean distance between each observation and its corresponding cluster center to $k$th power (where $k > 0$). This new clustering criterion, which we call the \textit{within-cluster sum-of-powers criterion}, can be written as follows:
\begin{equation}	
\mathcal{D}_{n,k} := \argmin_{\mathcal{D}: \# \mathcal{D}=n} \frac{1}{N} \sum_{j=1}^{N} \|\bm{x}_j - Q(\bm{x}_j; \mathcal{D})\|_2^k \;.
\label{eq:wcsp}
\end{equation}
Note that, for $k = 2$, the criterion in (\ref{eq:wcsp}) reduces to the k-means clustering criterion in (\ref{eq:kmean_obj_dis}). 

To aid the derivation of our method, we will assume for now that $N = \infty$, i.e., an infinite amount of data is available on the population. Letting $f(\cdot)$ be the data generating density function, the above clustering criterion can be generalized as:
\begin{equation}
\mathcal{D}_{n,k}^* := \argmin_{\mathcal{D}: \# \mathcal{D} = n} V_k(\mathcal{D};f),
\label{eq:optnk}
\end{equation}
where:
\begin{equation}
V_{k}(\mathcal{D}; f) := \bigg[ \int \|\bm{x} - Q(\bm{x}; \mathcal{D})\|_2^k f(\bm{x}) \; d\bm{x} \bigg]^{1/k}.
\label{eq:qerr}
\end{equation}
Here, `*' is used to denote the setting of infinite data ($N = \infty$). Note that the power $1/k$ is added to \eqref{eq:qerr}; this does not affect the optimization of the cluster centers (since the transformation is monotonic), but will aid in subsequent derivations.

The criterion in (\ref{eq:optnk}) has been studied in the signal processing literature, particularly for vector quantization. \cite{GL2007} showed that the cluster centers in \eqref{eq:optnk} converge asymptotically to a distribution with density $f^{p/(p+k)}$ as the number of cluster centers $n \rightarrow \infty$\footnote{From here on, the notion of asymptotics refers to the number of cluster centers $n \rightarrow \infty$.}. This suggests that as $k \rightarrow 0$, the density of $\mathcal{D}_{n,k}^*$ becomes closer to the data-generating density $f$. This motivates the following \textit{distributional clustering criterion}:
\begin{equation}
\mathcal{D}_{n,0}^* := \argmin_{\mathcal{D}: \# \mathcal{D} = n} V_0(\mathcal{D};f),
\label{eq:optn}
\end{equation}
where:
\begin{equation}
V_0(\mathcal{D};f) := \lim_{k \rightarrow 0^+} V_k(\mathcal{D}; f).
\label{eq:v0}
\end{equation}
The following theorem gives a closed-form expression for the limiting objective function, $V_{0}(\mathcal{D}; f)$, in (\ref{eq:v0}).

\begin{theorem}[Within-cluster sum-of-limiting-power]
For any $\mathcal{D} \subseteq \mathbb{R}^p$, $\# \mathcal{D} = n$, we have:
\begin{equation}
V_0(\mathcal{D};f) = \exp\left\{ \int \log \|\bm{x} - Q(\bm{x};\mathcal{D})\|_2 f(\bm{x}) \; d\bm{x} \right\}.
\label{eq:v02}
\end{equation}
\label{thm:v0}
\end{theorem}
\noindent The proof is given in Appendix A.

The key advantage of Theorem \ref{thm:v0} is that, as we show later in Section 4, it provides us a closed-form expression for the distributional clustering criterion (\ref{eq:optn}) in the finite data setting of $N < \infty$. We will call the log-term $\log \|\bm{x} - Q(\bm{x};\mathcal{D})\|_2$ in \eqref{eq:v02} as the \textit{log-potential} of point $\bm{x}$ with respect to center $Q(\bm{x};\mathcal{D})$; this is motivated by a similar log-potential distance used in the experimental design literature \citep{dette2010generalized}.

\setcounter{equation}{0} 
\section{Distributional convergence}
With the distributional clustering criterion (\ref{eq:optn}) in hand, we now show that the optimal cluster centers $\mathcal{D}_{n,0}^*$ indeed fixes the distribution distortion problem from k-means, i.e., the empirical distribution of these centers converges asymptotically to the underlying data distribution. Let $F(\bm{x})$, $\bm{x} \subseteq \mathbb{R}^p$ be the data-generating distribution function (assumed to be continuous), with density function $f(\bm{x})$. Let $F_{n,k}$ and $F_{n,0}$ denote the empirical distribution functions (e.d.f.s) for the optimal cluster centers in \eqref{eq:optnk} and \eqref{eq:optn}, respectively. Also, let $F_{\infty,k}(\bm{x})$ denote the distribution function with density proportional to $f^{p/(p+k)}(\bm{x})$. The following theorem shows that, under mild regularity conditions on $F$, the distribution of the proposed cluster centers, $\mathcal{D}_{n,0}^*$, converges asymptotically to the desired distribution $F$. 
\begin{theorem}[Distributional convergence of distributional clustering]
Suppose the data distribution $F$ satisfies the mild assumptions:
\begin{itemize}
\item [(\textbf{A1})] (Moment):
$\mathbb{E}\|\bm{X}\|^{k+\delta} < \infty$, $\bm{X} \sim F$ for some $\delta \geq 1$, and $k > 0$.
\item [(\textbf{A2})] (Uniform convergence):
$\lim\limits_{n \rightarrow \infty} \sup\limits_{\bm{x} \in \mathbb{R}^p} |F_{n,k}(\bm{x})-F_{\infty,k}(\bm{x})| = 0$ uniformly over $k \in [0,\xi)$ for some $\xi > 0$.
\end{itemize}
Then $F_{n,0}$, the e.d.f of $n$ distributional clustering centers, satisfies:
\begin{equation}
\lim_{n \rightarrow \infty} F_{n,0}(\bm{x}) = F(\bm{x}) \text{\quad for all \quad} \bm{x} \in \mathbb{R}^p.
\label{eq:conv}
\end{equation}
\label{thm:conv}
\end{theorem}
\noindent Assumption \textbf{(A1)} is a standard moment assumption on the data distribution $F$. Assumption \textbf{(A2)} concerns the uniform convergence of $F_{n,k}$ to $F_{\infty,k}$, and holds for densities $f$ which are not too heavy-tailed, for example, the exponential and power distributions \citep{FP2002}. Note that pointwise convergence of $F_{n,k}$ to $F_{\infty,k}$ is already proven in Theorem 7.5 of \cite{GL2007}. The proof for this theorem is given in Appendix B.

Theorem \ref{thm:conv} shows that, by minimizing the new clustering criterion in \eqref{eq:optn}, the resulting cluster centers capture the data distribution as the number of representative points $n \rightarrow \infty$. This shows that the proposed distributional clustering method indeed addresses the key problem of distribution distortion for k-means clustering.

\setcounter{equation}{0} 
\section{Cluster center optimization for distributional clustering}

Next, we present next an optimization algorithm for efficiently computing these new cluster centers in practice. This algorithm is described in three parts. First, we present a clustering algorithm for solving \eqref{eq:optn} in the practical setting of finite data (i.e., $N < \infty$). Second, we propose an adjustment which leads to better distribution representation in the case when the number of representative points $n$ is small. Third, we combine both parts to obtain a general optimization algorithm for distributional clustering.

\subsection{Optimizing the within-cluster sum-of-log-potential criterion}

We begin by presenting an optimization algorithm for computing the cluster centers in \eqref{eq:optn} in the case of finite data (i.e., $N < \infty$). Recall that Theorem \ref{thm:conv} shows the empirical distribution of $\mathcal{D}_{n,0}^*$ converges asymptotically to the underlying data distribution $F$. To obtain $\mathcal{D}_{n,0}^*$, we therefore need to solve the optimization problem in \eqref{eq:optn}. Since $f(\cdot)$ is typically unknown in practice, we approximate the integral in \eqref{eq:optn} with a finite sample average over the data we have available. This yields the following optimization problem:
\begin{equation}
\argmin_{\mathcal{D}: \# \mathcal{D} = n} \frac{1}{N} \sum_{j=1}^{N} \log \Big[ \left\Vert\left. \bm{x}_j-Q(\bm{x}_j; \mathcal{D}) \right\Vert\right._2 + \delta \Big].
\label{eq:monte}
\end{equation}
Here, a small nugget parameter $\delta > 0$ is added to \eqref{eq:monte} to avoid singularity of the objective function when $\bm{x}_j = Q(\bm{x}_j; \mathcal{D})$. Letting $\mathcal{D} = \{\bm{d}_1,...,\bm{d}_n\}$, where $\bm{d}_i$ is the $i$th cluster center to be optimized, then \eqref{eq:monte} can be written as:
\begin{equation}
\argmin_{\bm{d}_1,...,\bm{d}_n; \bm{W}} \sum_{i=1}^{n} \sum_{j=1}^{N} w_{ij} \log \Big[ \left\Vert\left. \bm{x}_j-\bm{d}_i \right\Vert\right._2 + \delta \Big],
\label{eq:clusters}
\end{equation}
where $w_{ij}$ is a binary decision variable indicating the assignment of point $j$ to cluster $i$, and $\bm{W} = {(w_{ij})_{i=1}^n}_{j=1}^N$ is the $n$ x $N$ matrix corresponding to $w_{ij}$'s. Note that the constant $1/N$ is removed from \eqref{eq:clusters} since it does not affect the optimization. 

The optimization problem in \eqref{eq:clusters} is now quite similar to the k-means criterion in \eqref{eq:kmean_obj_dis}. Thus, to optimize \eqref{eq:clusters}, we will use a novel extension of Lloyd's algorithm \citep{lloyd1982least}, which is widely used for optimizing the k-means clustering problem \eqref{eq:kmean_obj_dis}. The idea behind Lloyd's algorithm is as straight-forward. Beginning with an initial set of cluster centers (which are randomly sampled from the data), the first step is to (a) assign each data point to its nearest cluster center in Euclidean norm; this optimizes the assignment decision variables $\bm{W}$ given fixed centers. Next, for each cluster, (b) update its cluster center as the point which minimizes the sum of squared distances with all cluster points; this optimizes the cluster centers $\bm{d}_1, \cdots, \bm{d}_n$ given fixed cluster assignments. For \eqref{eq:kmean_obj_dis}, the latter step amounts to simply updating each cluster center as the mean of its cluster points. Both optimization steps are then repeated iteratively until convergence.

Since the proposed criterion in \eqref{eq:clusters} is similar to the k-means criterion in \eqref{eq:kmean_obj_dis}, we can extend Lloyd's iterative algorithm for optimizing \eqref{eq:clusters}. The key modification is in step (b) of the above description. Define first the log-potential of the $i$th cluster as:
\begin{equation}
LP_i(\bm{d}) = \sum_{j=1}^{N_i} \log \Big[ \left\Vert\left. \bm{x}_{ij}-\bm{d} \right\Vert\right._2 + \delta \Big],
\label{eq:lp}
\end{equation}
where, in the $i$th cluster, $N_i$ is the number of points and $\bm{x}_{ij}$ is its $j$th point. Note that the log-potential is defined given fixed cluster assignments. In view of \eqref{eq:clusters}, the only modification required in step (b) is that we instead update cluster centers by minimizing the log-potential in \eqref{eq:lp}. 


Unfortunately, unlike k-means, the minimum of the log-potential \eqref{eq:lp} has no closed-form expression. However, this minimization can be greatly simplified by the following intuition. Suppose we have a one-dimensional dataset, and the $i$th cluster contains three points: $x_{i1} = 0.1, x_{i2} = 0.4$ and $x_{i3} = 0.9$. Figure \ref{fig:log_pot} plots the log-potential $\sum_{j=1}^{3} \log [\left\Vert\left. x_{ij}-d_i \right\Vert\right._2 + \delta ]$ with a small nugget term $\delta = 0.01$. From this, the local minima of the log-potential appears to occur at the three data points within the cluster. Theorem \ref{thm:lp_minima} shows that this is indeed the case for sufficiently small $\delta$.
\begin{figure}[t]
\centering
\includegraphics[width=0.75\textwidth]{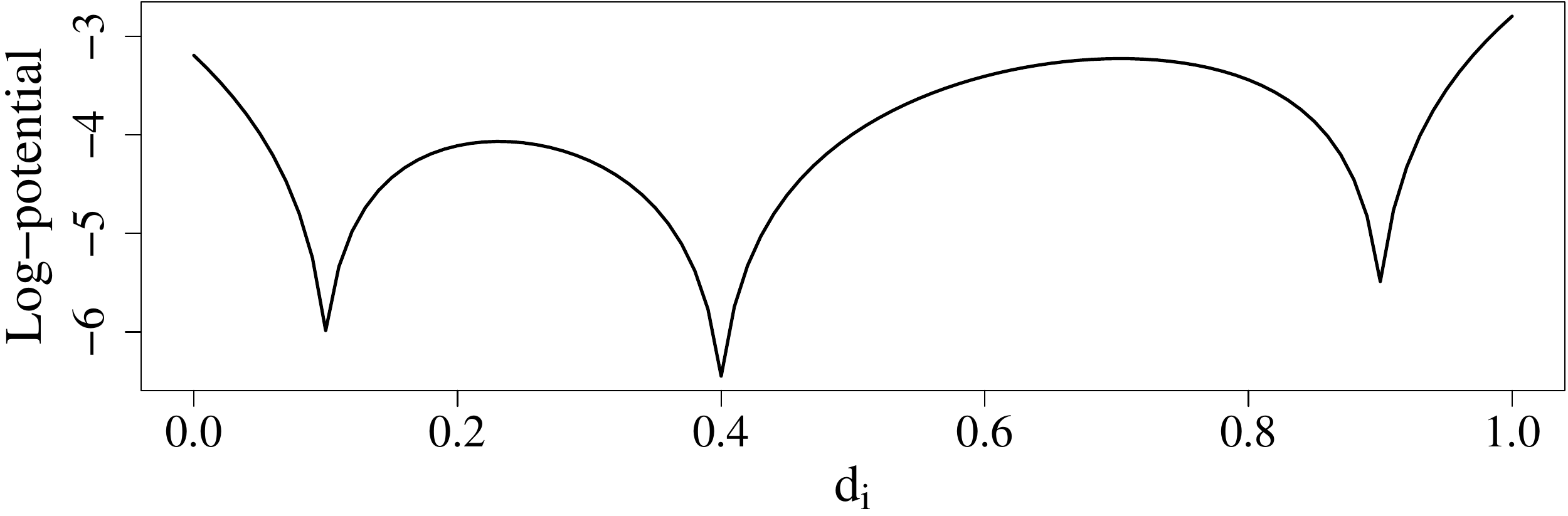}
\caption{Log-potential $LP(d)$ as a function of cluster center $d$.}
\label{fig:log_pot}
\end{figure}
\begin{theorem}[Log-potential minima]
For any $\mathcal{X} = \{\bm{x}_1,...,\bm{x}_N\}$, there exists a $\delta^*$, such that for all $\delta \in (0,\delta^*)$, the global minimum $\bm{d}^*$ of the log-potential $LP(\bm{d}) = \sum_{j=1}^{N} \log \Big[ \left\Vert\left. \bm{x}_j-\bm{d} \right\Vert\right._2 + \delta \Big]$ is found in $\mathcal{X}$.
\label{thm:lp_minima}
\end{theorem}
\noindent The proof is given in Appendix C.

The insight from this theorem is that, for each cluster, it allows us to search for the optimal center \textit{within} the points assigned to this cluster. In particular, using Theorem \ref{thm:lp_minima}, the optimization problem in \eqref{eq:clusters} can be written as:
\begin{equation} \label{eq_logd}
\argmin_{\bm{d}_1,...,\bm{d}_n;\bm{d}_i \in \mathcal{X}_i; {\bm{W}}} \sum_{i=1}^{n} \Big[\log \delta + \sum_{j = 1;\bm{x}_{j} \ne \bm{d}_i}^{N} w_{ij} \log \big[\left\Vert\left. \bm{x}_{j}-\bm{d}_i \right\Vert\right._2 + \delta \big]\Big],
\end{equation}
where $\mathcal{X}_i=\{\bm{x}_{j} | w_{ij}=1,1 \leq j \leq N\}$ is the set of points assigned to the $i$th cluster. The $\log \delta$ term in \eqref{eq_logd} results from the fact that $\bm{d}_i$ is also a data point in the $i$th cluster. Since $\log\ \delta$ is constant and $\delta \ll1$, (\ref{eq_logd}) can further be simplified to obtain the cluster centers $\mathcal{D}_{n,0}$ in the finite data setting ($N< \infty$) as:
\begin{equation} 
\mathcal{D}_{n,0} = \argmin_{\bm{d}_1,...,\bm{d}_n;\bm{d}_i \in \mathcal{X}_i; {\bm{W}}} \sum_{i=1}^{n} \sum_{j = 1;\bm{x}_{j} \ne \bm{d}_i}^{N} w_{ij} \log \left\Vert\left. \bm{x}_{j}-\bm{d}_i \right\Vert\right._2.
\label{eq:eq_dc_rep}
\end{equation}

Hence, given fixed assignment variables $\bm{W}$, each cluster center $\bm{d}_i$ can then be updated by the discrete optimization problem:
\begin{equation} 
\bm{d}_i \leftarrow \argmin_{\bm{d} \in \mathcal{X}_{i}} \sum_{j = 1;\bm{x}_{ij} \ne \bm{d}}^{N_i} \log \left\Vert\left. \bm{x}_{ij}-\bm{d} \right\Vert\right._2, \quad i = 1, \cdots, n,
\label{eq:dc_asymp_step2}
\end{equation}
where $N_i$ is the number of points and $\bm{x}_{ij}$ is the $j$th point, assigned to the $i$th clutser ($w_{ij} = 1$). This can be efficiently optimized via several heuristics, which we discuss next.

We now summarize the full optimization algorithm for \eqref{eq:eq_dc_rep} (called \texttt{DC\_asymp}) in Algorithm 1. First, initial cluster centers are randomly sampled from the data. Next, repeat the following two steps until convergence: (a) assign each data point to its nearest cluster center, and (b) update each cluster center to minimize the log-potential criterion \eqref{eq:dc_asymp_step2}. {Although step (b) solves a discrete optimization problem, its computation time can be greatly reduced using the following heuristic. The key idea is that we do not need to compute the objective over every data point in the cluster (or $\bm{d} \in \mathcal{X}_{i}$). Since the log-potential cluster center in \eqref{eq:dc_asymp_step2} measures cluster centrality, one would expect it to be near the sample mean of all the points in the cluster. Thus, one way to reduce computation is to compute the objective for only the $r\%$ points closest to the mean (for a small choice of $r$) and screen out the remaining $(1-r)\%$ of points. The choice of $r = 10\%$  seems to work well in our numerical examples.

\begin{algorithm}[t]
    Sample initial cluster centers ${\{\bm{d}_i\}}_{i=1}^{n}$ from the data $\mathcal{X}$, and set $\mathcal{D}_{n,0} \leftarrow {\{\bm{d}_i\}}_{i=1}^{n}$\\   
    \While {$\mathcal{D}_{n,0}$ does not converge}{
     {Step 1: Assign} each data point in $\mathcal{X}$ to nearest cluster under $\left\Vert\left.. \right\Vert\right._2$\\
     {Step 2: Update} cluster centers ${\{\bm{d}_i\}}_{i=1}^{n}$ to $\mathcal{D}_{n,0}$ by solving \eqref{eq:dc_asymp_step2} for each cluster\\
}   	
\caption{\texttt{DC\_asymp $(\mathcal{X}, n)$}\label{A1}}
\end{algorithm}

We also note an interesting connection between distributional clustering and k-medoids clustering \citep{kaufman1987clustering}. The cluster centers in k-medoids optimize the same criterion as k-means, but are restricted to be within the data itself (which is similar to \texttt{DC\_asymp}). However, distributional clustering enjoys the same advantage over k-medoids as it does over k-means: the proposed cluster centers capture the data distribution asymptotically, whereas this is not true for k-medoids clustering. Hence, the proposed method is expected to perform better in applications where cluster prototypes should be representative of the data distribution.

\subsection{Optimizing the within-cluster sum-of-powers criterion}

Note that Theorem \ref{thm:conv} guarantees the proposed cluster centers converge in distribution to the data distribution as the number of centers $n \rightarrow \infty$. In the case of small $n$, however, an adjustment can be performed to provide improved distribution representation. To see why, Fig. \ref{fig:finite_n_effect_n} (left) shows the kernel density estimate of cluster centers on 1-dimensional standard normal data for different values of $n$, with $k$ fixed at 0. As $n$ decreases, the optimal cluster centers tend to move towards the high-density regions of the data, which makes the estimated density function less heavy tailed. One reason for this is that, as we minimize the log-potential between the cluster center and cluster points, there is lesser incentive to minimize the distance between the cluster center from a far-off point than in minimizing it from many nearby points. Thus, for small $n$, cluster centers are pushed towards high-density regions, which induces a finite-sample distribution distortion (note that this distortion disappears as $n \rightarrow \infty$). We refer to this distortion as a ``small-$n$ distortion''.

\begin{figure}[t]
\centering
\includegraphics[width=1\textwidth]{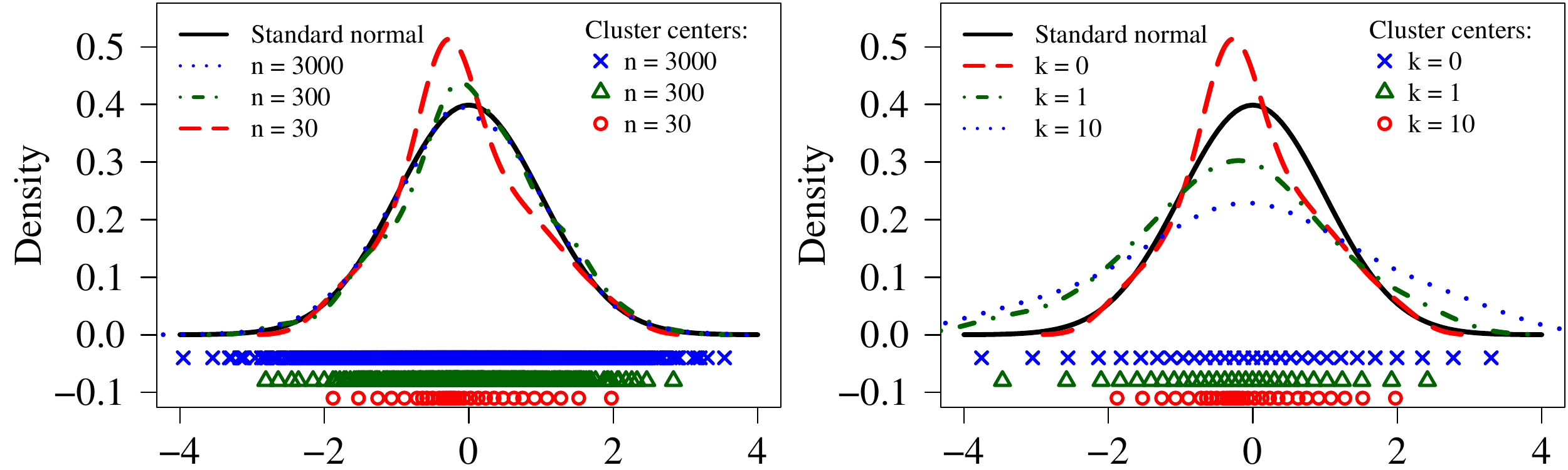}
\captionsetup{width=1\textwidth}
\caption{(left): Effect of $n$ on distribution of cluster centers for $k = 0$; (right): Effect of $k$ on distribution of cluster centers for $n = 30$.}
\label{fig:finite_n_effect_n}
\end{figure}

Recall that the within-cluster sum-of-log-potential criterion \eqref{eq:eq_dc_rep} is a limiting case of the within-cluster sum-of-powers criterion \eqref{eq:wcsp}. To correct the small-$n$ distortion in the former, one strategy is to optimize the latter with an appropriately chosen power $k$.  Figure \ref{fig:finite_n_effect_n} (right) shows the kernel density estimate of the optimal cluster centers on 1-dimensional standard normal data for different values of $k$, with $n$ fixed at 30. As $k$ increases, the density of the cluster centers becomes more heavy tailed, which is opposite of the effect of the small-$n$ distortion! This suggests that the small-$n$ distortion can be corrected by increasing power $k$ in equation \eqref{eq:wcsp}. Intuitively, one reason for this is that, as $k$ increases, the incentive is higher in minimizing the distance of the cluster center from a far away point rather than from many nearby points. This then encourages cluster centers to move towards low density regions, which leads to cluster centers being more spread out. 

Adopting the above adjustment strategy, we present next an optimization algorithm, called \texttt{DC\_finite} (see Algorithm 2), for the within-cluster sum-of-powers criterion \eqref{eq:wcsp} with fixed $k$ (a tuning procedure for $k$ is discussed in the next section). First, initial cluster centers are randomly sampled from the data. Next, the following two steps are iterated until convergence: (a) assign each data point to its nearest cluster center, and (b) update each cluster center to minimize the sum-of-powers within each cluster. In particular, step (b) performs the following update:
\begin{equation} 
\bm{d}_i = \argmin_{\bm{d}} \sum_{j = 1}^{N_i} \left\Vert\left. \bm{x}_{ij}-\bm{d} \right\Vert\right._2^k, \quad i = 1, \cdots, n.
\label{eq:dc_finite_step2}
\end{equation}
For $k \geq 1$, the optimization in \eqref{eq:dc_finite_step2} is convex, so a global optimum can be obtained via the truncated Newton method \citep{dembo1983truncated}. In our implementation, this optimization is performed using the R package `nloptr' \citep{ypma2014introduction}. A similar approach was used by \cite{mak2018minimax} for the case of $k$ large, but for a different goal of experimental design.

\begin{algorithm}[t]
   Sample initial cluster centers ${\{\bm{d}_i\}}_{i=1}^{n}$ from the data $\mathcal{X}$, and set $\mathcal{D}_{n,k} \leftarrow {\{\bm{d}_i\}}_{i=1}^{n}$\\ 
    \While {$\mathcal{D}_{n,k}$ does not converge}{
    {Step 1: Assign} each data point in $\mathcal{X}$ to nearest cluster under $\left\Vert\left.. \right\Vert\right._2$\\
   {Step 2: Update} cluster centers ${\{\bm{d}_i\}}_{i=1}^{n}$ to $\mathcal{D}_{n,k}$ by solving \eqref{eq:dc_finite_step2} for each cluster\\
}
\caption{\texttt{DC\_finite $(\mathcal{X}, n, k)$}\label{A2}}
\end{algorithm}

\subsection{Distributional clustering algorithm}

Finally, we present a way to tune the power $k$ in \eqref{eq:wcsp} to best correct the small-$n$ distortion. We will make use of the following energy distance \citep{szekely2004testing}. Given data $\mathcal{X} = {\{\bm{x}_j\}}_{j=1}^{N}$ and optimal cluster centers $\mathcal{D}_{n,k} = {\{\bm{d}_i\}}_{i=1}^{n}$, the energy distance between $\mathcal{X}$ and $\mathcal{D}_{n,k}$ is defined as:
\begin{equation} \label{energy_emp}
\begin{aligned}
E(\mathcal{X}, \mathcal{D}_{n,k}) = \frac{2}{nN} \sum_{i=1}^{n} \sum_{j=1}^{N} \left\Vert\left. \bm{x}_j - \bm{d}_i \right\Vert\right._2 -  \frac{1}{N^2}\sum_{i=1}^{N}  \sum_{j=1}^{N} \left\Vert\left. \bm{x}_i - \bm{x}_j \right\Vert\right._2 - \\ 
\frac{1}{n^2}\sum_{i=1}^{n}  \sum_{j=1}^{n} \left\Vert\left. \bm{d}_i - \bm{d}_j \right\Vert\right._2.
\end{aligned}
\end{equation}
\noindent This energy distance was initially proposed as a two-sample goodness-of-fit test between two datasets $\mathcal{X}$ and $\mathcal{D}_{n,k}$. Here, we do not use this criterion for goodness-of-fit testing, but rather to tune a good choice of power $k$ which maximizes goodness-of-fit between data $\mathcal{X}$ and cluster centers $\mathcal{D}_{n,k}$. More specifically, we wish to find the power $k^*$ which satisfies:
\begin{equation} \label{optimal_k}
k^* = \argmin_k E(\mathcal{X}, \mathcal{D}_{n,k}).
\end{equation}

In implementation, $k^*$ is estimated as follows. First, beginning with the initial case of $k=0$, we generate the optimal cluster centers $D_{n,0}$ using the algorithm \texttt{DC\_asymp} in Section 4.1, and compute its energy distance to data $\mathcal{X}$. Next, we iteratively increase power $k$ by $\Delta = 0.5$, starting from $k=1$, then generate the optimal cluster centers $D_{n,k}$ using the algorithm \texttt{DC\_finite} in Section 4.2, and compute its energy distance to the data. We increment $k$ as long as the computed energy distance decreases, and terminate the procedure when it increases for a new power $k$. Finally, we take the optimal $k^*$ as the power prior to an increase in energy distance. From simulations (see Figure \ref{fig:optimalk}), the energy distance $E(\mathcal{X}, \mathcal{D}_{n,k})$ appears to be near-convex in $k$, which justifies this iterative tuning approach. Algorithm 3, which we call \texttt{DC}, outlines the full distributional clustering algorithm with power tuning.


Figure \ref{fig:optimalk} visualizes this tuning procedure for two toy examples. The left plot shows the plot of energy distance against $k$ for a 10-dimensional standard normal data with $N = 100,000$, and $n = 100$. Here, the tuned power is $k^* = 15$. In light of the discussion in Section 4.2, this large power is not surprising, since the number of representative points $n=100$ is quite small. The right plot shows the tuned power $k^*$ as a function of $n$, where the data is generated from a 9-dimensional multivariate standard normal distribution with $N=90,000$. As $n$ increases, the tuned power $k^*$ needed to correct this distortion decreases to zero, which makes sense since the small-$n$ distortion disappears as $n \rightarrow \infty$. This also supports the result in Theorem 2, that the distributional clustering centers converge to the data distribution as $n \rightarrow \infty$.

\begin{figure}[t]
\centering
\includegraphics[width=0.9\textwidth]{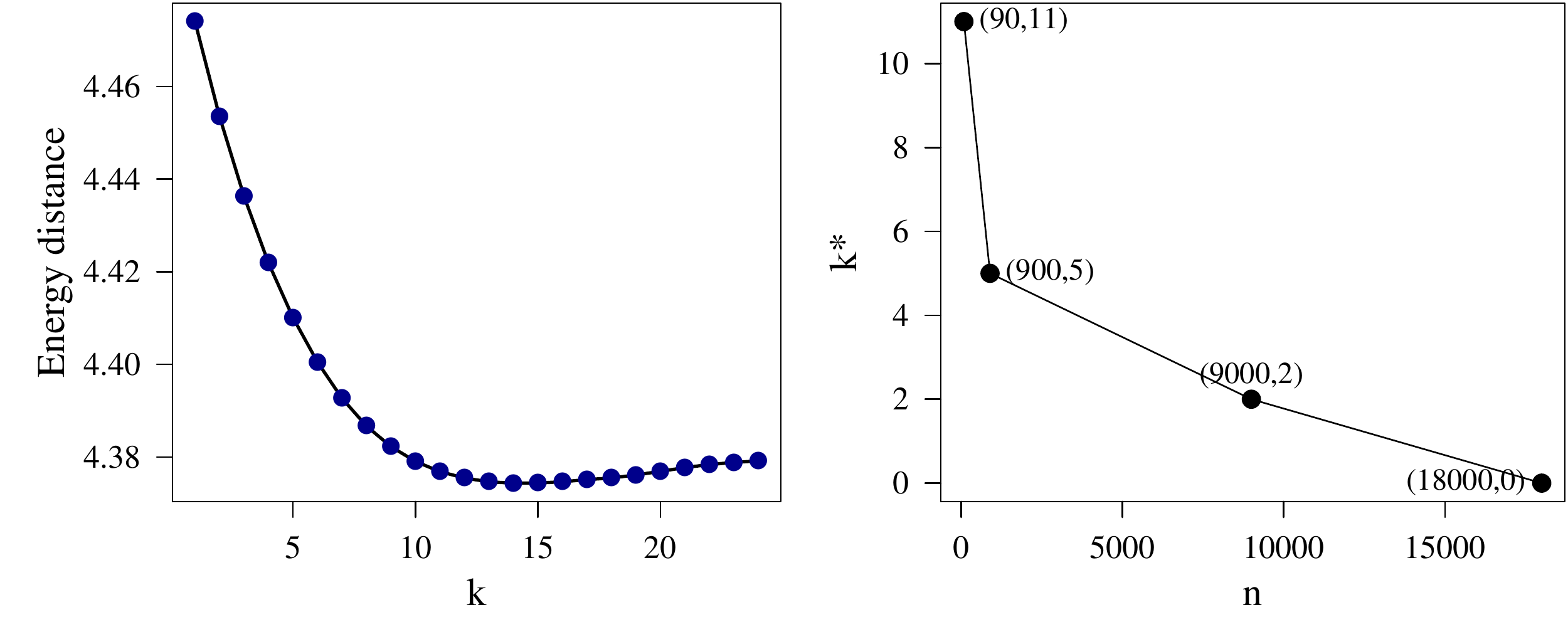}
\caption{(left): The energy distance $E(\mathcal{X}, \mathcal{D}_{n,k})$ as a function of power $k$; (right): The optimal power $k^*$ as a function of the number of clusters $n$.}
\label{fig:optimalk}
\end{figure}

\begin{algorithm}
 $\mathcal{D}_{n,0} \leftarrow \texttt{DC\_asymp}(\mathcal{X}, n) $\\
 $E_{cur} \leftarrow E(\mathcal{X},\mathcal{D}_{n,0})$\\
 $E_{prev} \leftarrow E_{cur} + 1$\\
 $k \leftarrow 1$\\
 \While {$E_{cur}<E_{prev}$}{
  $E_{prev} \leftarrow E_{cur}$\\
  $\mathcal{D}_{n,k} \leftarrow \texttt{DC\_finite}(\mathcal{X}, n, k)$\\
 $E_{cur} \leftarrow E(\mathcal{X},\mathcal{D}_{n,k})$\\
 $k \leftarrow k + \Delta$\\
}
 Output $\mathcal{D}_{n,k-\Delta}$
\caption{\texttt{DC $(\mathcal{X}, n,\Delta)$}\label{A3}}
\end{algorithm}

\setcounter{equation}{0} 
\section{Numerical examples}
We now investigate the performance of distributional clustering in two numerical examples. To provide a fair comparison, we will use the energy distance as well as another metric -- the multivariate Cram\'er statistic \citep{baringhaus2004new} -- to compare different reduction methods. The Cram\'er statistic between the data $\mathcal{X} = {\{\bm{x}_j\}}_{j=1}^{N}$, and cluster centers $\mathcal{D}_{n,k} = {\{\bm{d}_i\}}_{i=1}^{n}$ (for a given $k$) is:
\begin{equation} \label{cramer_stat}
\begin{aligned}
C(\mathcal{X}, \mathcal{D}_{n,k}) = \frac{nN}{N+n} \Bigg(\frac{2}{nN} \sum_{i=1}^{n} \sum_{j=1}^{N} \phi(\left\Vert\left. \bm{x}_j - \bm{d}_i \right\Vert\right._2^2) -  \frac{1}{N^2}\sum_{i=1}^{N}  \sum_{j=1}^{N} \phi(\left\Vert\left. \bm{x}_i - \bm{x}_j \right\Vert\right._2^2) - \\
\frac{1}{n^2}\sum_{i=1}^{n}  \sum_{j=1}^{n} \phi(\left\Vert\left. \bm{d}_i - \bm{d}_j \right\Vert\right._2^2)\Bigg),
\end{aligned}
\end{equation}
where $\phi$ is a kernel function. We have chosen $\phi(z) =  1-\exp(-z/2)$, as this kernel compares the distributions on both dispersion and location. The lower the Cram\'er-statistic, the closer the respective distributions.

\begin{example}
We compare the numerical performance of our distributional clustering algorithm with k-means clustering and random subsampling, on data simulated from the standard normal, exponential and gamma distributions, with dimensions varying from 2 to 8. Here, $N=100n$, and $n = 10p$. The energy distance in \eqref{energy_emp}, and the multivariate Cram\'er statistic in \eqref{cramer_stat} are used as metrics for evaluating the reduction methods. Figure \ref{fig:ex1} shows the energy distance and Cram\'er statistic for each of the three reduction methods, over the three distribution choices. We see that the converged cluster centers for distributional clustering have both the lowest energy distance and the lowest Cram\'er statistic, for all distributions and dimensions, which suggests that the proposed clustering method better captures the distribution of the underlying data compared to existing methods. In case of $p = 5$ for normal distribution, the energy distance is almost the same for $k=2$ and $k=3$, and so we observe a comparable performance of distributional clustering with k-means. For the tuned power $k^*$, we also observe larger $k^*$ for the normally-distributed data (with values increasing in dimension), but smaller $k^*$ for the exponential and gamma-distributed data, including $k^*=0$ for some cases. This suggests, while the clustering procedure under the within-cluster sum-of-log-potential criterion is asymptotically consistent for distribution matching, it may also be useful for problems with a small number of clusters $n$ (depending on the distribution of the data). However, in other cases, our other procedure is also important to ensure good distribution matching.

\begin{figure}[t]
\centering
\includegraphics[width=0.95\textwidth]{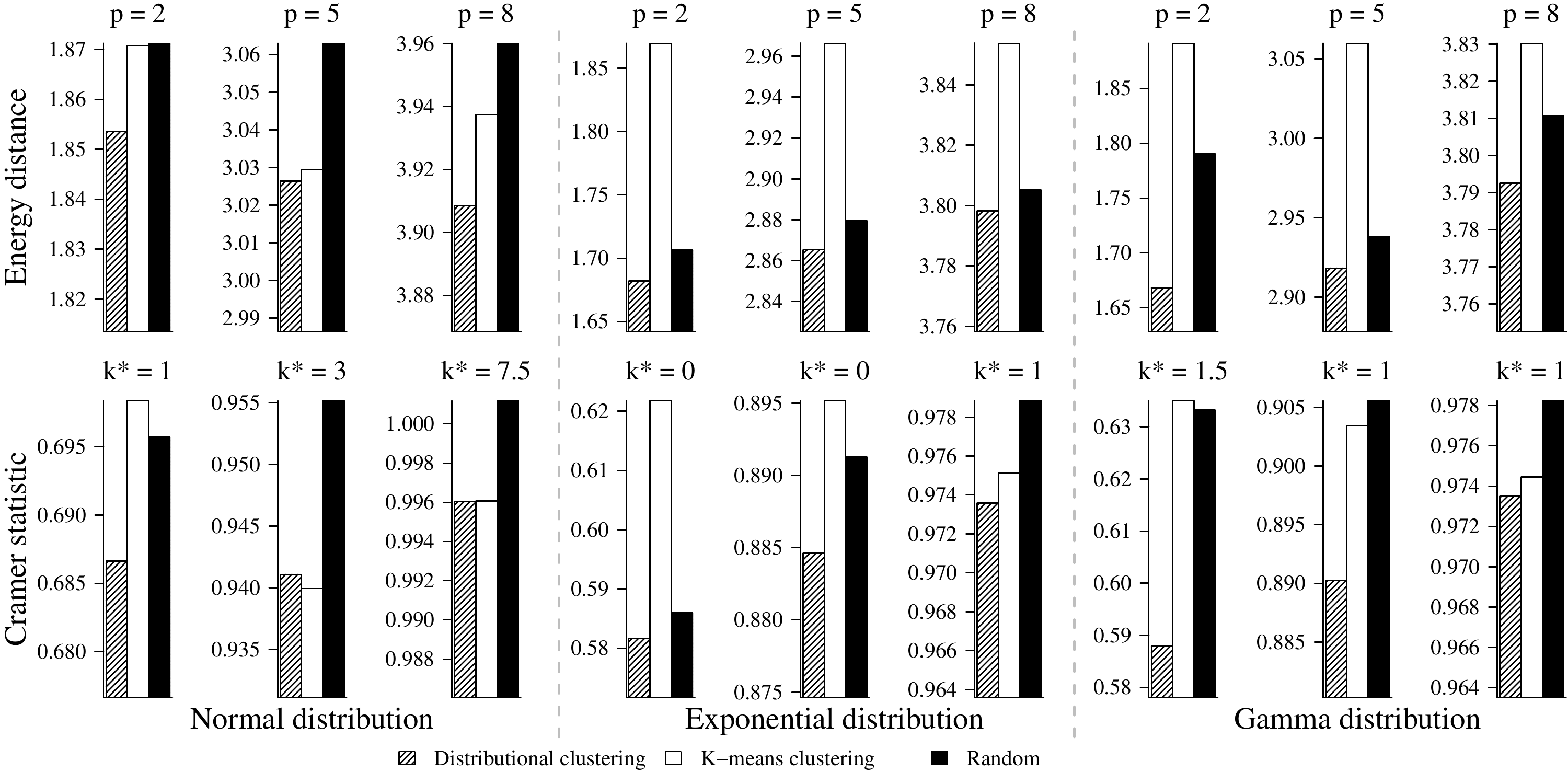}
\caption{Comparison of energy distances (top) and Cram\'er statistics (bottom) of distributional clustering, k-means and random subsampling, for standard normal, standard exponential and gamma (shape = 1, rate = 1) distributions.}
\label{fig:ex1}
\end{figure}
\end{example}

\begin{example}
High resolution regional climate models (RCMs) are driven from a comparatively coarse resolution global climate models for depicting high resolution future climate states of the atmosphere with associated uncertainty. Several targeted random sampling techniques have been developed (for example, \cite{rife2013selecting}) to select a subset of days from the coarse resolution dataset such that distribution of climate variables on those days matches with that of the entire population. Using such a representative sample provides a more economical and computationally feasible method of determining climate change uncertainty \citep{pinto2014regional}. Thus, we are motivated to apply distributional clustering to reduce climate data.

We demonstrate the effectiveness of distributional clustering on climate data (https://rattle.togaware.com/weatherAUS.csv) containing daily measurements of wind speed, humidity,  pressure and minimum temperature in Australia from October 2007 to June 2017. Here, $N=100,000$ and $n=100$. In order to show the effectiveness of distributional clustering on not just one sample, but in general for any sample, we randomly select 100 distinct initializing samples, and compare performance of all reduction methods on each of the 100 samples. Figure \ref{fig:ex3} shows the energy distance and Cram\'er statistic boxplots for the reduced samples from distributional clustering, k-means clustering, and random subsampling. We see that both the energy distances and Cram\'er statistics for distributional clustering are noticeably lower than those for the other two methods, which suggests that the proposed method again outperforms existing methods in terms of capturing the distribution of the underlying data. 


\begin{figure}[t]
\centering
\includegraphics[width=0.7\textwidth]{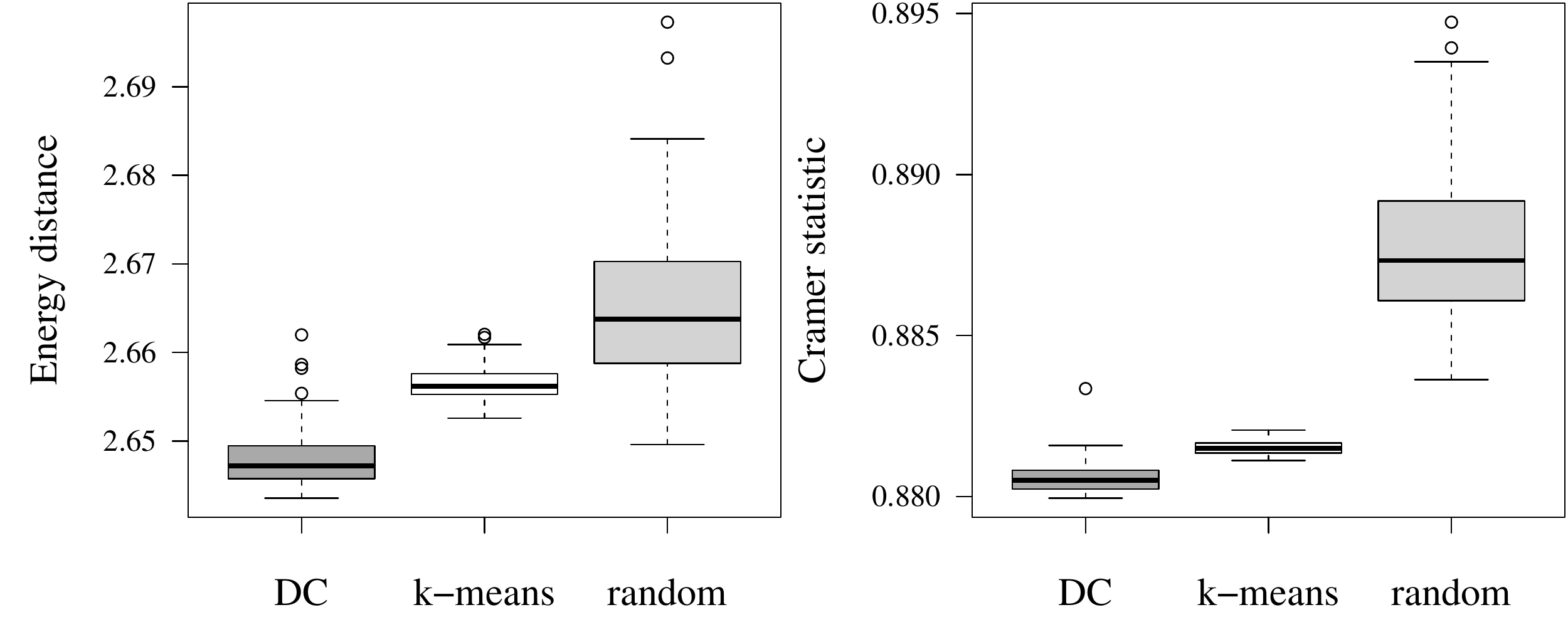}
\caption{Boxplots of the energy distances (left) and Cram\'er statistics (right) for 100 randomly chosen initializing samples for distributional clustering (DC), $k$-means and the initializing sample.}
\label{fig:ex3}
\end{figure}
\end{example}

\section{Conclusion}
K-means clustering is a widely-used approach for identifying cluster prototypes of an underlying dataset. One drawback, however, is that the k-means cluster centers (as representative points) distort the distribution of the data. To address this, we proposed a new distributional clustering method, which ensures the cluster centers indeed capture the data distribution. We proved the asymptotic convergence of the proposed cluster centers to the data generating distribution, then presented an efficient algorithm for computing such centers in practice. Numerical examples show that the proposed cluster prototypes provide a better representation of the underlying data, compared to k-means clustering and random subsampling. 

There are several interesting directions to pursue for future research. First, while the representative points from distributional clustering preserve the overall data distribution, it does not necessarily preserve \textit{marginal} distributions over each variable -- a property shown to be important for high-dimensional data reduction \citep{mak2017projected}. An extension of the proposed method for capturing marginal distributions would be worthwhile. Second, it would be nice to develop a more computationally efficient strategy for tuning the power $k^*$, which would allow our method to scale better for large datasets.



\vskip 14pt
\noindent {\large\bf Acknowledgements}

This research is supported by a U.S. National Science Foundation grant DMS-1712642.
\par


\vskip 14pt

\setcounter{section}{0}
\setcounter{equation}{0}
\setcounter{theorem}{0}
\def\theequation{A.\arabic{equation}}
\def\thesection{A\arabic{section}}

\noindent {\large\bf Appendix A: Proof of Theorem 1.}

\begin{theorem}[Within-cluster sum-of-limiting-power]
For any $\mathcal{D} \subseteq \mathbb{R}^p$, $\# \mathcal{D} = n$, we have:
\begin{equation}
V_0(\mathcal{D};f) = \exp\left\{ \int \log \|\bm{x} - Q(\bm{x};\mathcal{D})\|_2 f(\bm{x}) \; d\bm{x} \right\}.
\label{eq:v02_proof}
\end{equation}
\label{thm:v0_proof}
\end{theorem}
\begin{proof}\textit{(Theorem \ref{thm:v0_proof})}.
We apply log transformation to (2.3), and then apply limit on $k$ to obtain $\lim_{k \rightarrow 0^+}  \log V_k(\mathcal{D}; f) = \lim_{k \rightarrow 0^+} [\{\log \int \|\bm{x} - Q(\bm{x};\mathcal{D})\|_2 f(\bm{x}) \; d\bm{x}\}/k]$. Using L'Hospital's Rule to compute the limit, we obtain $\lim_{k \rightarrow 0^+}  \log V_k(\mathcal{D}; f) = \lim_{k \rightarrow 0^+} [\{\int \|\bm{x} - Q(\bm{x};\mathcal{D})\|_2\ \log \|\bm{x} - Q(\bm{x};\mathcal{D})\|_2 f(\bm{x}) \; d\bm{x}\}/\{\int \|\bm{x} - Q(\bm{x};\mathcal{D})\|_2 f(\bm{x}) \; d\bm{x}\}]$. Computing the limit for $k \rightarrow 0^+$, we obtain (2.6), which completes the proof.
\end{proof}

\noindent {\large\bf Appendix B: Proof of Theorem 2.}
\def\theequation{B.\arabic{equation}}
\setcounter{equation}{0}

\begin{theorem}[Distributional convergence of distributional clustering]
Suppose $F$ satisfies the mild assumptions \textbf{(A1)} and \textbf{(A2)}. Then $F_{n,0}$, the e.d.f of $n$ distributional clustering centers, satisfies:\\
\newtheorem{assumpA1}{(A1)}
\newtheorem{assumpA2}{(A2)}

\begin{equation}
\lim_{n \rightarrow \infty} F_{n,0}(\bm{x}) = F(\bm{x}), \quad \forall \bm{x} \in \mathbb{R}^p.
\label{eq:conv_proof}
\end{equation}
\label{thm:conv_proof}
\end{theorem}

\begin{proof}\textit{(Theorem \ref{thm:conv_proof})}. 
First, we claim that for any $k$, (i) $\lim_{n \rightarrow \infty}F_{n,k}(\bm{x}) = F_{\infty,k}(\bm{x})$ for all $\bm{x} \in \mathbb{R}^p$. This follows directly from Theorem 7.5 in \cite{GL2007}, under the assumption (A1). This convergence is also uniform in $k \in [0,\xi)$ under assumption (A2). Next, applying Scheff\'e's lemma \citep{Sch1947}, it is easy to see that (ii) $\lim_{k \rightarrow 0^+}F_{\infty,k}(\bm{x}) = F(\bm{x})$ for all $\bm{x} \in \mathbb{R}^p$, since $\lim_{k \rightarrow 0^+} f^{p/(p+k)} (\bm{x})$ (the limiting density of $F_{\infty,k}$) converges to $f(\bm{x})$ (the density of $F$) almost everywhere in $\bm{x}$.\\

We now wish to use (i) and (ii) to prove \eqref{eq:conv_proof}. Choose any event $A \in \mathcal{B}(\mathbb{R}^p)$, the Borel $\sigma$-algebra on $\mathbb{R}^p$. Let $\tilde{p}_{n,k}$, $\tilde{p}_{\infty,k}$, $\tilde{p}_{n,0^+}$ and $\tilde{p}$ denote the probability of $A$ under $F_{n,k}$, $F_{\infty,k}$, $F_{n,0}$ and $F$, respectively. If the following lemma holds:
\begin{lemma}[Exchange of limits]
\begin{equation}
\lim_{k \rightarrow 0^+} \left\{ \lim_{n \rightarrow \infty} \tilde{p}_{n,k} \right\} = \lim_{n \rightarrow \infty} \left\{ \lim_{k \rightarrow 0^+} \tilde{p}_{n,k} \right\},
\label{eq:limex_proof}
\end{equation}
\label{lem:limex_proof}
\end{lemma}
\noindent then equation \eqref{eq:conv_proof} must hold, because the left-hand side equals $\tilde{p}$ by applying (i) and (ii), and the right-hand side implies \eqref{eq:conv_proof}.

\begin{proof}\textit{(Lemma \ref{lem:limex_proof})}.
From (i) and (ii), the iterated limit $\lim_{k \rightarrow 0^+} (\lim_{n \rightarrow \infty}\tilde{p}_{n,k}) = \tilde{p}$. Also, $\lim_{n \rightarrow \infty}\tilde{p}_{n,k}$ exists uniformly for $k \in [0,\xi)$ under assumption (A2). Then, by Theorem 2.15 of \cite{habil2016double}, the double limit $\lim_{k \rightarrow 0^+, n \rightarrow \infty}\tilde{p}_{n,k} = \tilde{p}$. Since $\lim_{k \rightarrow 0^+}\tilde{p}_{n,k}$ exists, and $\lim_{n \rightarrow \infty}\tilde{p}_{n,k}$ exists (from (i)), then, by Theorem 2.13 of \cite{habil2016double}, the iterated limit $\lim_{n \rightarrow \infty} (\lim_{k \rightarrow 0^+}\tilde{p}_{n,k}) = \tilde{p}$, which completes the proof.
\end{proof}
\end{proof}

\noindent {\large\bf Appendix C: Proof of Theorem 3.}
\def\theequation{D.\arabic{equation}}
\setcounter{equation}{0}
\begin{theorem}[Log potential minima]
For any $\mathcal{X} = \{\bm{x}_1,...,\bm{x}_N\}$, there exists a $\delta^*$, such that for all $\delta \in (0,\delta^*)$, the global minimum $\bm{d}^*$, of the log-potential $LP(\bm{d}) = \sum_{j=1}^{N} \log \Big[ \left\Vert\left. \bm{x}_j-\bm{d} \right\Vert\right._2 + \delta \Big]$, is found in $\mathcal{X}$.
\begin{proof}\textit{(Theorem \ref{thm:lp_minima_proof})}.
Log-potential at $\bm{d} = \bm{x}_k$, where $\bm{x}_k \in \{\bm{x}_1,...,\bm{x}_N\}$, is $LP(\bm{x}_k) = \log \delta  + \sum_{j \neq k; j = 1}^{N} \log \Big[ \left\Vert\left. \bm{x}_j-\bm{x}_k \right\Vert\right._2 + \delta \Big] $. Log potential at $\bm{d} = \bm{d'}$, where $\bm{d'} \not\in \{\bm{x}_1,...,\bm{x}_N\}$, is $LP(\bm{d'}) = \sum_{j=1}^{N} \log \Big[ \left\Vert\left. \bm{x}_j-\bm{d'} \right\Vert\right._2 + \delta \Big] $. Let $\sum_{j \neq k;j=1}^N \log \Big[ \left\Vert\left. \bm{x}_j-\bm{x}_k \right\Vert\right._2 + \delta \Big] = a_k, LP(\bm{d'}) = a_{d'}$, where $a_k \in \mathbb{R}$ and $a_{d'} \in \mathbb{R}$. Then $LP(\bm{x}_k) - LP(\bm{d'}) = \log \delta + a_k - a_{d'}$. For any $a_k \in \mathbb{R}$ and $a_{d'} \in \mathbb{R}$, there exists a $\delta_k \in (0,\delta_k^*)$ such that $\log \delta_k + a_k - a_{d'} < 0$, or $LP(\bm{x}_k) < LP(\bm{d'})$. Let $\delta^* = min(\delta_1^*,..., \delta_n^*)$. Then, for $\delta \in (0,\delta^*),  LP(\bm{x}_k) < LP(\bm{d'}) \ \forall \ k = 1,..., n$. This implies that all points in $\mathcal{X} = \{\bm{x}_1,...,\bm{x}_N\}$ correspond to local minima, or one of the points in $\mathcal{X}$ is the global minimum.
\end{proof}
\label{thm:lp_minima_proof}
\end{theorem}

\bibliography{references}  


\end{document}